\documentclass{article}

\usepackage{microtype}
\usepackage{graphicx}
\usepackage{booktabs}
\usepackage{hyperref}

\usepackage{configuration/conference,times}

\usepackage{amsmath,amssymb,mathtools,amsthm}
\usepackage{booktabs,multirow,array,tabularx,longtable}
\usepackage{float}
\usepackage{needspace}
\usepackage{wrapfig}
\usepackage{xcolor}
\usepackage{colortbl}
\definecolor{bestcolor}{RGB}{213,232,212}
\definecolor{secondcolor}{RGB}{255,242,204}
\definecolor{thirdcolor}{RGB}{230,224,248}
\newcommand{\cbest}[1]{\cellcolor{bestcolor}#1}
\newcommand{\csecond}[1]{\cellcolor{secondcolor}#1}
\newcommand{\cthird}[1]{\cellcolor{thirdcolor}#1}
\usepackage{pifont}
\usepackage{xspace}
\usepackage{tikz}
\usetikzlibrary{arrows.meta,positioning,fit}
\usepackage{algorithm}
\usepackage{algorithmic}
\usepackage{url}
\usepackage[capitalize,noabbrev]{cleveref}
\hypersetup{
    colorlinks=true,
    linkcolor=blue,
    citecolor=blue,
    urlcolor=blue
}

\usepackage{amsmath,amsfonts,bm}

\def\1{\bm{1}}

\DeclareMathAlphabet{\mathsfit}{\encodingdefault}{\sfdefault}{m}{sl}
\SetMathAlphabet{\mathsfit}{bold}{\encodingdefault}{\sfdefault}{bx}{n}

\newcommand{\R}{\mathbb{R}}

\usepackage{xurl}

\newcommand{\method}{ER-Chord\xspace}
\newcommand{\base}{ChordEdit\xspace}
\newcommand{\cmark}{\ding{51}}
\newcommand{\xmark}{\ding{55}}

\newcommand{\csrc}{c_{\mathrm{src}}}
\newcommand{\ctar}{c_{\mathrm{tar}}}
\newcommand{\uc}{u_{\mathrm C}}
\newcommand{\uer}{u_{\mathrm{ER}}}
\newcommand{\Ec}{E_{\mathrm C}}
\newcommand{\Mismatch}{\mathcal{M}}
\newcommand{\leadheading}[1]{\par\medskip\noindent\textbf{#1}\enspace\ignorespaces}

\newtheorem{theorem}{Theorem}[section]

\newtheorem{proposition}[theorem]{Proposition}
\newtheorem{corollary}[theorem]{Corollary}
\theoremstyle{definition}

\theoremstyle{remark}

\title{ER-Chord: Spatially Consistent One-Step Image Editing via Equal-Energy Reliability Redistribution}

\author{%
\parbox[t]{\linewidth}{\raggedright
\mbox{\textbf{Peize Li}$^{1,2,*}$}\quad\mbox{\textbf{Hanzhen Zhao}$^{1,*}$}\quad\mbox{\textbf{Xingguo Xu}$^{3,*}$}\quad\mbox{Yilong Fan$^{4}$}\quad\mbox{Zhen Tao$^{5}$}\quad\mbox{Bo Yin$^{1}$}\quad\mbox{Zhifei Liu$^{1}$}\quad\mbox{Qingdong He$^{6}$}\quad\mbox{Xiaobin Hu$^{1,\dagger}$}\\[8pt]
$^{1}$National University of Singapore\quad $^{2}$Kuaishou Technology\quad $^{3}$Dalian University of Technology\\[2pt]
$^{4}$Tianjin University\quad $^{5}$Harbin Institute of Technology\quad $^{6}$Tencent\\[6pt]
$^{*}$Equal contribution\quad $^{\dagger}$Corresponding author}
}

\iclrfinalcopy

\newcommand\blfootnote[1]{%
  \begingroup
  \renewcommand\thefootnote{}\footnote{#1}%
  \addtocounter{footnote}{-1}%
  \endgroup
}

\begin{document}

\maketitle
\lhead{}
\blfootnote{Preprint.}

\begin{abstract}
One-step generative models enable fast text-guided image editing, yet they often introduce unintended changes to regions that should remain unchanged. Temporal aggregation stabilizes editing fields across noise levels, but leaves their spatial allocation insensitive to local variations in reliability. Consequently, temporal stability alone does not ensure spatially consistent editing. To address this limitation, we propose \method, a training-free and inversion-free framework for reliability-guided spatial redistribution. Our key insight is that agreement between editing fields at neighboring noise levels provides a local reliability signal for allocating editing strength. We formulate this allocation as an equal-energy optimization that redistributes control energy according to local reliability while preserving the original editing directions and total control energy. This formulation admits a unique global solution that can be efficiently computed using existing field evaluations, requiring no additional denoiser calls. Experiments on two benchmarks, PIE-Bench and TEdBench, demonstrate improvements over state-of-the-art one-step editing baselines, achieving superior source preservation alongside consistent gains across multiple generative backbones.The project page is at \url{https://github.com/Peizeli1/ER-Chord}.
\end{abstract}

\section{Introduction}

Distilled text-to-image models, such as SD-Turbo, InstaFlow, and SwiftBrush-v2~\citep{sauer2024adversarial,liu2024instaflow,dao2024swiftbrush}, compress diffusion sampling into a single neural function evaluation. Although this acceleration enables real-time image synthesis, editing real images remains difficult because it requires introducing requested semantic changes while preserving the identity, structure, and unedited regions of the source image. Existing accelerated editors address this trade-off through one-step inversion~\citep{nguyen2025swiftedit}, few-step reconstruction~\citep{deutch2024turboedit,gong2025instantedit}, or inversion-free transport~\citep{kulikov2025flowedit,lu2026chordedit}. Among these approaches, training-free and inversion-free transport is especially attractive because it eliminates both auxiliary model training and iterative latent inversion. However, executing an edit through a single transport update imposes strict requirements on the spatial organization of the resulting control field.

ChordEdit~\citep{lu2026chordedit} stabilizes one-step editing by aggregating source-to-target proxy fields from nearby noise levels. This formulation primarily addresses temporal variations in the estimated transport field. Because this aggregation rule is applied uniformly across the latent grid, the magnitude at each spatial location is merely a passive result of combining the endpoint fields. However, the two endpoint estimates do not agree equally across all locations. Some regions exhibit consistent local control vectors across noise levels, whereas others show substantial disagreement. Uniform temporal aggregation ignores this spatial disparity and distributes editing strength without considering local reliability. Consequently, it often introduces large updates into unedited backgrounds, causing visual artifacts. Therefore, temporal stability does not ensure spatial consistency, leaving spatial magnitude allocation as a crucial yet unaddressed degree of freedom.

This observation motivates a distinction between \emph{where the estimated control is reliable} and \emph{where a semantic edit should occur}. Existing localized editing methods typically use cross-attention maps or internal feature representations to identify edit regions~\citep{couairon2023diffedit,simsar2025lime,soni2025locatedit,nguyen2025swiftedit,hu2026whereedit}. In contrast to these localization approaches, we retain the temporally aggregated editing directions and use endpoint agreement solely to determine their relative spatial strengths. Thus, endpoint consistency serves as an intrinsic reliability signal for the control field rather than a spatial localization map.

Based on this principle, we propose \method, a training-free and inversion-free framework for reliability-guided spatial redistribution. \method preserves the local directions of the Chord field and maintains the total control energy, modifying only how this fixed budget is distributed across spatial locations. We measure local reliability through the agreement between the two Chord endpoint fields and formulate the reallocation as a reliability-weighted quadratic optimization problem subject to an exact equal-energy constraint. This constraint prevents trivial source preservation caused by simply reducing the global editing strength. Although this formulation couples all spatial locations through a shared energy budget, it yields a unique global solution that rescales each local Chord vector by a positive scalar. The entire solution depends on a single scalar variable that can be found efficiently through one-dimensional bisection. \method therefore requires no additional denoiser evaluations, auxiliary training, or backpropagation.

\begin{figure*}[t]
\includegraphics[width=\linewidth]{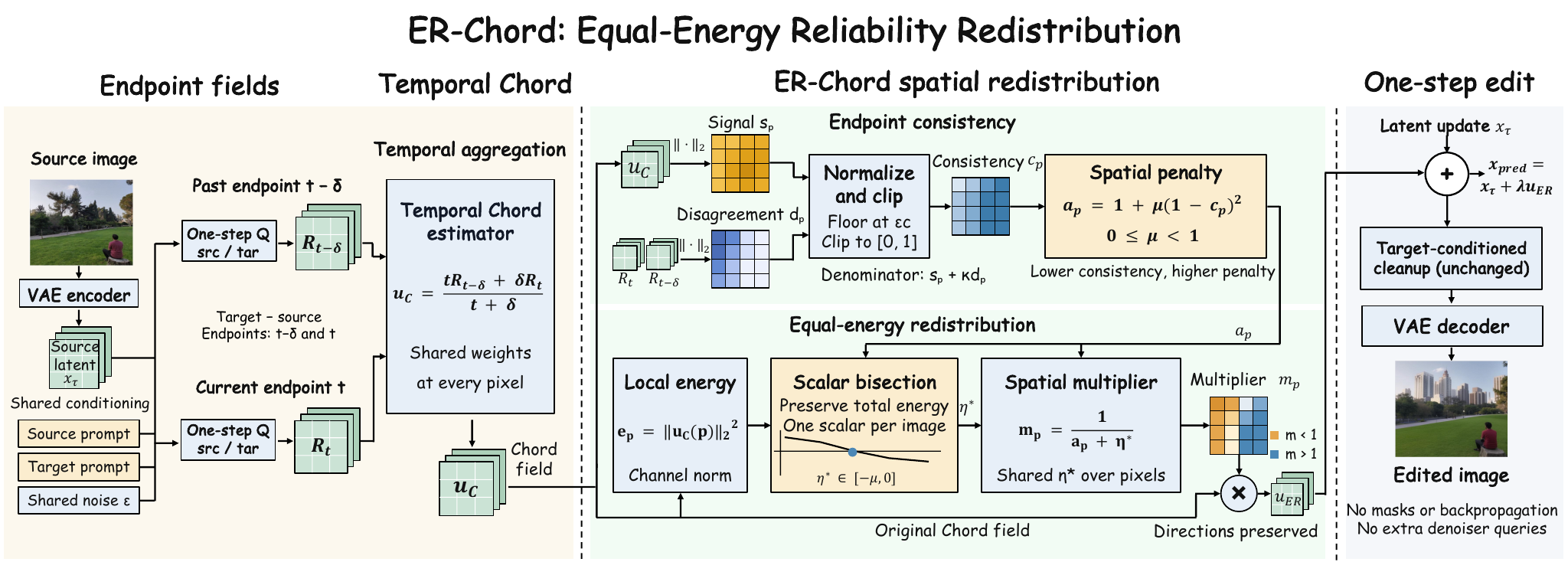}
\caption{\textbf{Method overview.} Starting from a temporally aggregated control field, \method measures local endpoint consistency and redistributes a fixed control energy budget while preserving local directions. The framework requires no auxiliary guidance, training, inversion, additional denoiser evaluations, or backpropagation.}
\label{fig:method_overview}
\end{figure*}

The equal-energy formulation also provides a direct interpretation of the resulting redistribution. Locations with higher endpoint consistency receive larger scaling factors relative to their original Chord vectors, while a shared scalar coordinates these adjustments to preserve the total control energy exactly. Hence, \method modifies only the relative spatial allocation of energy without altering local control directions or the overall update scale (Figure~\ref{fig:method_overview}). As shown in Figures~\ref{fig:temporal-vs-spatial} and~\ref{fig:spatial-energy-redistribution}, under identical total control energy, ChordEdit disperses editing strength into irrelevant backgrounds, causing color bleeding and artifacts. In contrast, \method concentrates energy into semantically consistent regions, producing clean target edits while faithfully preserving source structures.

Empirically, \method consistently outperforms one-step editing baselines across diverse benchmarks using the same denoiser budget. Without tuning hyperparameters for specific datasets or backbones, our configuration substantially improves background preservation on PIE-Bench, generalizes effectively to other benchmarks such as TEdBench in perceptual quality and source fidelity, and transfers across multiple generative backbones (InstaFlow, SwiftBrush-v2, and SD-Turbo). These results establish reliability-guided spatial redistribution as an effective, orthogonal complement to temporal stabilization.

Our main contributions are summarized as follows:
\begin{itemize}
    \item We identify spatial magnitude allocation as an unaddressed degree of freedom in temporally aggregated one-step editing, demonstrating that temporal stability alone does not guarantee spatial consistency.

    \item We propose \method, a framework that uses endpoint consistency as an intrinsic reliability signal to redistribute total control energy under an exact equal-energy constraint, which is solved globally through a one-dimensional scalar root.

    \item We demonstrate on PIE-Bench, TEdBench, and three generative backbones that \method consistently improves source preservation while maintaining semantic edit fidelity, requiring no external guidance, training, or additional denoiser calls.
\end{itemize}

\section{Related Work}

\begin{figure}[t]
\centering
\includegraphics[width=0.98\linewidth]{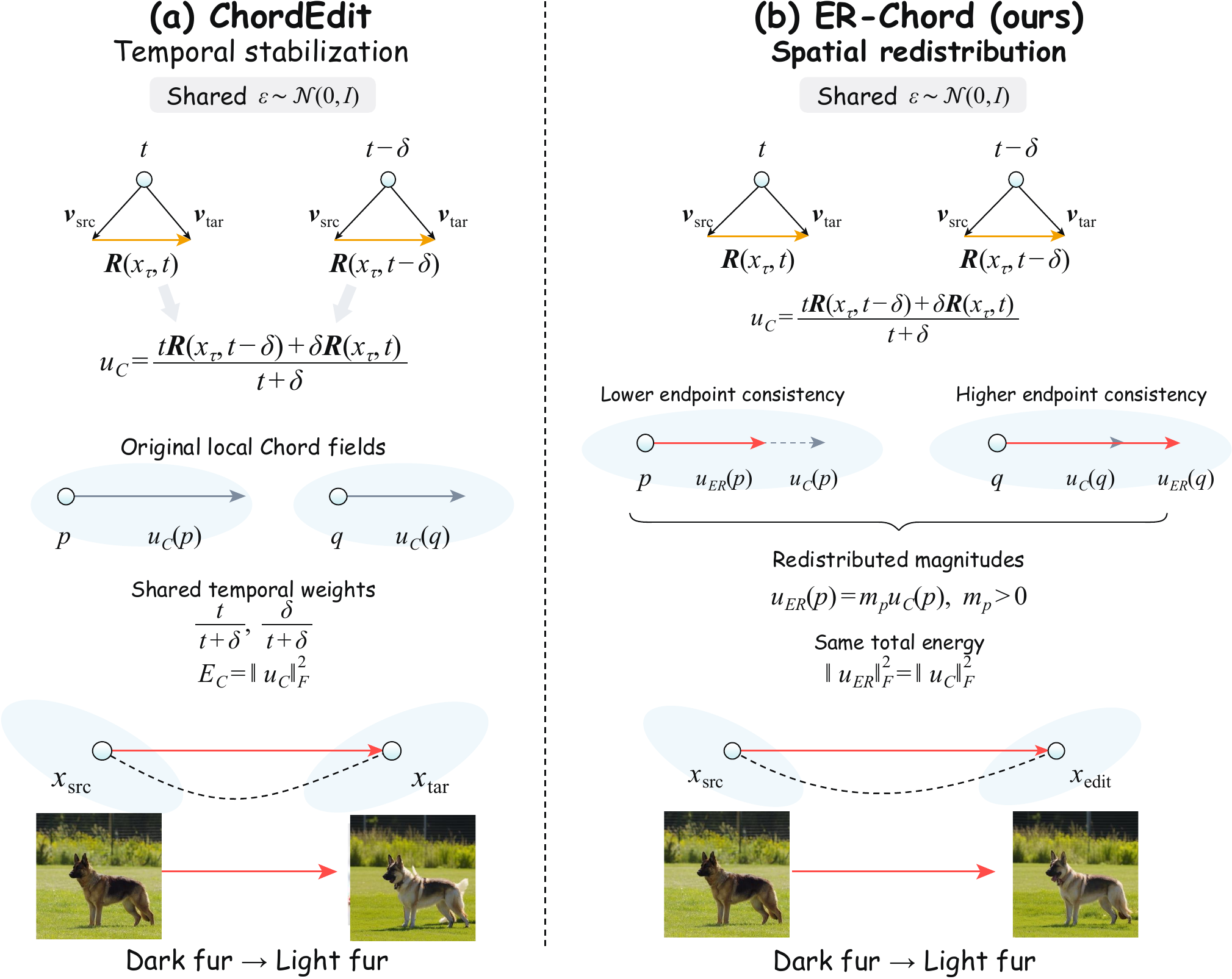}
\caption{\textbf{Temporal vs.\ spatial control in one-step editing.} (a)~\base temporally aggregates endpoint fields with uniform weights across space, which fails to reflect local differences in prediction consistency and disperses edit energy into irrelevant backgrounds. (b)~\method measures local endpoint consistency as a reliability signal and redistributes control energy ($u_{\mathrm{ER}}(p) = m_p u_{\mathrm{C}}(p)$) under an exact conserved energy budget ($\|u_{\mathrm{ER}}\|_F^2 = \|u_{\mathrm{C}}\|_F^2$), concentrating updates onto semantically intended target regions without background artifacts.}
\label{fig:temporal-vs-spatial}
\end{figure}

\leadheading{Real-image editing and inversion.}
Text-guided editing must introduce the requested semantics without losing the identity and layout of the source image. SDEdit starts reverse diffusion from a partially noised source, while DiffEdit estimates an edit mask from differences between conditional predictions~\citep{meng2022sdedit,couairon2023diffedit}. Prompt-to-Prompt controls cross-attention, and MasaCtrl and PnP reuse self-attention or intermediate features from a source branch to preserve structure~\citep{hertz2023prompt,cao2023masactrl,tumanyan2023plug}. Inversion-based methods instead seek a latent trajectory that accurately reconstructs the source. Null-text Inversion optimizes unconditional text embeddings, EDICT introduces coupled transformations for exact inversion, and subsequent methods reduce the mismatch accumulated along inversion and denoising trajectories~\citep{mokady2023null,wallace2023edict,ju2024pnp,huberman2024edit,yang2026directedit}. A related line of work separates inversion and invariance within flow transformers to enable versatile editing~\citep{xu2025unveil}. These approaches improve fidelity by modifying the inversion path, internal features, or spatial masks. \method instead operates after a one-step source-to-target control field has been formed and changes only its local magnitudes, without reconstructing an inversion trajectory or modifying network attention and features.

\needspace{6\baselineskip}
\leadheading{Few-step and one-step editing.}
Consistency-based and flow-based models have made high-quality synthesis possible with one or a few neural function evaluations~\citep{song2024improved,luo2023latent,lipman2023flow,liu2023flow}. This development has motivated editors designed specifically for low-step samplers. TurboEdit and InstantEdit pair specialized inversion with few-step reconstruction, whereas SwiftEdit learns a one-step inversion model and uses mask-guided attention rescaling for localized editing~\citep{deutch2024turboedit,gong2025instantedit,nguyen2025swiftedit}. BiFM jointly learns forward and inverse flows, FlowEdit constructs an inversion-free transport from pretrained flow models, and EEdit accelerates inversion-based editors by exploiting redundant computation across space and time~\citep{dai2026bifm,kulikov2025flowedit,yan2025eedit}. Related acceleration efforts span frequency-energy constrained routing for diffusion adaptation~\citep{yin2025fera} and future-guided frame selection for long-horizon video generation~\citep{yin2026framemorrow}. ChordEdit is the closest starting point for our work. It is training-free and inversion-free, and stabilizes a one-step update by aggregating source-to-target fields across nearby noise levels~\citep{lu2026chordedit}. A subsequent analysis interprets its chord window largely as an effective timestep shift and separates its procedure into coarse transport and fine alignment~\citep{li2026rethinking}. \method leaves this temporal estimator and cleanup unchanged, introduces no additional model evaluations, and only redistributes the resulting field magnitude over space.

\begin{wrapfigure}{r}{0.52\textwidth}
\centering
\includegraphics[width=\linewidth]{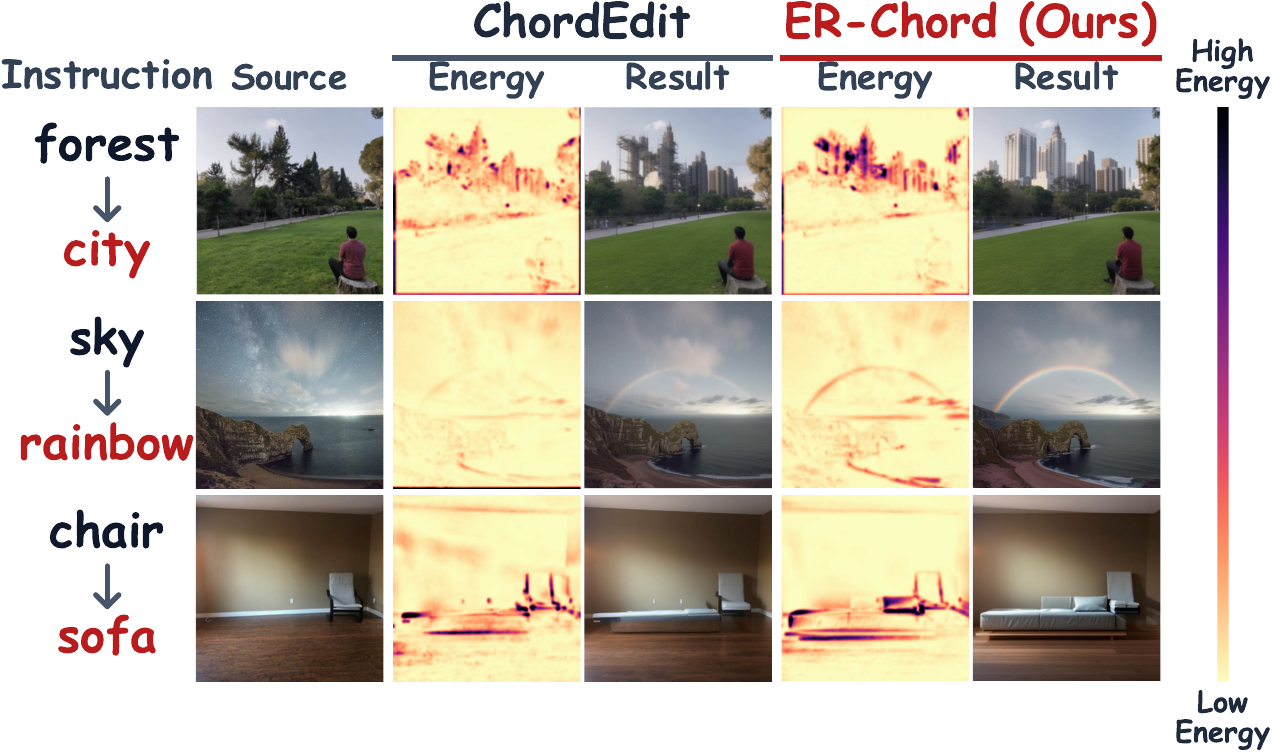}
\caption{\textbf{Spatial energy redistribution under conserved total energy.} Under identical conserved control energy ($1.00\times$), \base disperses energy into backgrounds causing artifacts, while \method concentrates energy on target regions for clean edits.}
\label{fig:spatial-energy-redistribution}
\end{wrapfigure}

\leadheading{Spatial localization and magnitude allocation.}
Spati-ally precise editors generally identify where an edit belongs and then restrict generation in that region. DiffEdit derives a mask from conditional prediction differences, while Learnable Regions and ZONE obtain regions through learned boxes or segmentation~\citep{couairon2023diffedit,lin2024text,li2024zone}. FoI, LIME, and LOCATEdit extract or refine attention-based regions and use them to constrain the denoising process~\citep{guo2024focus,simsar2025lime,soni2025locatedit}. ALE-Edit further limits attribute leakage through object-restricted text embeddings, region-guided attention masking, and background blending~\citep{mun2025addressing}. Localization has also been incorporated into one-step editing. SwiftEdit uses mask-guided attention rescaling, while WhereEdit discovers editable regions from internal features and applies local modulation~\citep{nguyen2025swiftedit,hu2026whereedit}. These methods use semantic or structural evidence to decide where a modification should occur. \method neither predicts an editable region nor changes the semantic direction of the edit. It uses endpoint consistency only to rescale the magnitudes of fixed local Chord directions under an unchanged per-image energy budget. Localization methods determine where an edit should act, whereas \method adjusts the strength of the existing Chord field across locations.

\section{Method}
\label{sec:method}
We develop \method in three steps. We first revisit the temporally aggregated
Chord field and identify spatial magnitude allocation as a degree of freedom
that is not explicitly controlled by the temporal estimator. We then use local
endpoint consistency as a reliability signal and formulate spatial reweighting
as an equal-energy redistribution problem, which preserves both the original
control directions and the total control strength. Finally, we show that this
seemingly coupled optimization admits a simple direction-preserving solution
determined by a single shared scalar, enabling an efficient implementation with
negligible additional computation.

\subsection{From Temporal Aggregation to Spatial Allocation}

Let $I_{\mathrm{src}}$ denote the source image, and let $c_{\mathrm{src}}$ and
$c_{\mathrm{tar}}$ denote the source and target conditions, respectively.
Given a latent state $\mathbf{x}_{\tau}$, we follow ChordEdit~\cite{lu2026chordedit}
and denote by $\mathbf{R}(\mathbf{x}_{\tau}, t)$ the source-to-target proxy field
constructed at noise time $t$. In particular, for the two neighboring endpoints
$t-\delta$ and $t$, we write
$\mathbf{R}_{t-\delta}:=\mathbf{R}(\mathbf{x}_{\tau}, t-\delta)$ and
$\mathbf{R}_{t}:=\mathbf{R}(\mathbf{x}_{\tau}, t)$.

ChordEdit temporally aggregates these two endpoint fields to obtain a stabilized
control field. At each spatial location $p\in\Omega$, its control vector can be
written as
\begin{equation}
\mathbf{u}_{\mathrm{C}}(p)
=
\frac{t}{t+\delta}\mathbf{R}_{t-\delta}(p)
+
\frac{\delta}{t+\delta}\mathbf{R}_{t}(p).
\label{eq:chord_field}
\end{equation}
The coefficients in \eqref{eq:chord_field} are global temporal scalars and do
not depend on the spatial location $p$. The corresponding derivation from the
Chord surrogate is provided in Appendix~\ref{app:chord_derivation}.

Equation~\eqref{eq:chord_field} provides a temporally stabilized editing vector
at every spatial location. However, the resulting spatial magnitude
$\|\mathbf{u}_{\mathrm{C}}(p)\|_2$ is only an implicit consequence of the
temporal aggregation. In particular, the aggregation rule does not explicitly
compare the magnitude assigned to location $p$ with the agreement between
$\mathbf{R}_{t-\delta}(p)$ and $\mathbf{R}_{t}(p)$. Therefore, two locations
with substantially different endpoint consistency are still governed by the
same temporal mixing rule.

This observation suggests separating temporal stabilization from spatial
magnitude allocation. We retain the local directions supplied by
$\mathbf{u}_{\mathrm{C}}$ and treat their spatial magnitudes as the remaining
degree of freedom. Our goal is to redistribute these magnitudes according to
local endpoint consistency while preserving the total control strength of the
original Chord field. Section~\ref{sec:redistribution} develops this
reliability-guided equal-energy redistribution.

\subsection{Reliability-Guided Equal-Energy Redistribution}
\label{sec:redistribution}

We next determine how the spatial magnitudes of the Chord field should be
reallocated. The key signal is the agreement between the two endpoint fields.
At a spatial location $p\in\Omega$, we use
$s_p:=\|\mathbf{u}_{\mathrm{C}}(p)\|_2$ to denote the magnitude of the
temporally aggregated Chord control and
$d_p:=\|\mathbf{R}_{t}(p)-\mathbf{R}_{t-\delta}(p)\|_2$ to measure the
disagreement between the two endpoint estimates. Since the same absolute
endpoint discrepancy can have different significance at locations with
different control magnitudes, we normalize the disagreement relative to the
local Chord signal.

Specifically, we define the local endpoint-consistency score as
\begin{equation}
c_p
=
\operatorname{clip}_{[0,1]}
\left(
\frac{s_p}
{\max\{s_p+\kappa d_p,\varepsilon_c\}}
\right),
\qquad
\kappa\geq 0,\ \varepsilon_c>0.
\label{eq:consistency_score}
\end{equation}
A larger $c_p$ indicates stronger agreement between the two endpoint fields
relative to the local control magnitude. In particular, $c_p$ approaches one
when the endpoint disagreement is small compared with $s_p$, and decreases as
the relative disagreement grows. This score measures local temporal
consistency of the estimated control field rather than semantic correctness.

We use this reliability signal only to redistribute spatial magnitude. The
local Chord directions remain unchanged as the reference editing directions,
and the total control strength is kept fixed. The latter constraint is
important because otherwise improved source preservation could be obtained
trivially by globally weakening the edit. Let
$E_{\mathrm{C}}:=\|\mathbf{u}_{\mathrm{C}}\|_F^2$ denote the total
pre-cleanup control energy of the Chord field. We seek a redistributed field
that remains close to $\mathbf{u}_{\mathrm{C}}$, while assigning a larger cost
to control magnitude placed at locations with lower endpoint consistency.

These design requirements lead to the following equal-energy redistribution:
\begin{equation}
\begin{aligned}
\mathbf{u}_{\mathrm{ER}}
\in
\arg\min_{\mathbf{u}}\quad
&
\frac{1}{2}
\sum_{p\in\Omega}
\|\mathbf{u}(p)-\mathbf{u}_{\mathrm{C}}(p)\|_2^2
+
\frac{\mu}{2}
\sum_{p\in\Omega}
(1-c_p)^2
\|\mathbf{u}(p)\|_2^2
\\
\text{subject to}\quad
&
\|\mathbf{u}\|_F^2
=
E_{\mathrm{C}},
\qquad
0\leq \mu<1.
\end{aligned}
\label{eq:er_optimization}
\end{equation}
The first term preserves fidelity to the original Chord field, while the
second term discourages control energy from remaining at locations with low
endpoint consistency. The equality constraint keeps the total pre-cleanup
energy exactly unchanged. Consequently, any reduction of control magnitude at
one subset of spatial locations must be compensated by an increase elsewhere.

Although \eqref{eq:er_optimization} couples all spatial locations through a
global energy constraint, its structure is considerably simpler than a generic
constrained field optimization. In the next subsection, we show that its
unique global solution preserves every local Chord direction and reduces the
entire redistribution to a one-dimensional scalar solve.

\subsection{Scalar-Root Solution and Properties}
\label{sec:closed_form}

The equal-energy constraint in \eqref{eq:er_optimization} couples all spatial
locations through a single global budget. Nevertheless, the optimization
admits a particularly simple structure. For each $p\in\Omega$, define
$a_p:=1+\mu(1-c_p)^2$, which represents the local quadratic penalty induced by
the endpoint-consistency score. The spatial dependence is entirely contained
in $a_p$, while the global energy constraint is enforced by one shared scalar.

\begin{theorem}[Unique direction-preserving solution]
\label{thm:er_solution}
Assume $0\leq \mu<1$ and $E_{\mathrm{C}}>0$. Then the optimization problem
\eqref{eq:er_optimization} admits a unique global minimizer. In particular,
there exists a unique scalar $\eta^\star\in[-\mu,0]$ such that
\begin{equation}
\mathbf{u}_{\mathrm{ER}}(p)
=
\frac{\mathbf{u}_{\mathrm{C}}(p)}
{a_p+\eta^\star},
\qquad p\in\Omega,
\label{eq:er_closed_form}
\end{equation}
where $\eta^\star$ is the unique root of
\begin{equation}
\sum_{p\in\Omega}
\frac{\|\mathbf{u}_{\mathrm{C}}(p)\|_2^2}
{(a_p+\eta)^2}
=
E_{\mathrm{C}}.
\label{eq:dual_root}
\end{equation}
\end{theorem}

Theorem~\ref{thm:er_solution} reduces the spatially coupled optimization to a
one-dimensional scalar solve. The local reliability scores determine the
spatially varying denominators $a_p$, while the single shared variable
$\eta^\star$ adjusts all locations jointly so that the original control energy
is exactly preserved. Since $a_p+\eta^\star>0$, the solution in
\eqref{eq:er_closed_form} rescales each local Chord vector without changing its
direction. The proof of existence, uniqueness, and global optimality is
provided in Appendix~\ref{app:er_solution_proof}.

To make the redistribution explicit, let
$m_p:=(a_p+\eta^\star)^{-1}$ denote the local spatial multiplier. Then
$\mathbf{u}_{\mathrm{ER}}(p)=m_p\mathbf{u}_{\mathrm{C}}(p)$. The following
properties summarize the behavior of the resulting field.

\begin{corollary}[Energy and direction preservation]
\label{cor:er_properties}
Under the assumptions of Theorem~\ref{thm:er_solution}, the \method solution
satisfies
\begin{equation}
\|\mathbf{u}_{\mathrm{ER}}\|_F^2
=
\|\mathbf{u}_{\mathrm{C}}\|_F^2.
\label{eq:energy_preservation}
\end{equation}
Moreover, $m_p>0$ for every $p\in\Omega$, and
\[
\frac{1}{1+\mu}
\leq
m_p
\leq
\frac{1}{1-\mu}.
\]
If $c_p>c_q$ and $\mu>0$, then $m_p>m_q$.
\end{corollary}

The corollary clarifies the role of the redistribution. \method neither rotates
the local Chord directions nor changes the total pre-cleanup control strength.
Instead, it modifies only the relative spatial magnitudes. Locations with
higher endpoint consistency receive a larger multiplicative adjustment than
locations with lower endpoint consistency. This ordering applies to the
multipliers themselves rather than to the absolute final energy shares, which
also depend on the original Chord magnitudes.

The scalar $\eta^\star$ can be obtained efficiently by bisection over
$[-\mu,0]$. For a latent field of size $C\times H\times W$ and $K$ bisection
iterations, the additional computation is $\mathcal{O}(K \cdot C \cdot H \cdot W)$ and consists
only of channel reductions and element-wise arithmetic. \method reuses the
endpoint fields already required by ChordEdit and therefore introduces no
additional denoiser evaluations.

To further characterize the spatial redistribution, we compare the fraction of
the total control energy assigned to each location before and after \method.
For $p\in\Omega$, define the original and redistributed spatial energy shares
as $\pi^{\mathrm{C}}_p:=\|\mathbf{u}_{\mathrm{C}}(p)\|_2^2/E_{\mathrm{C}}$
and
$\pi^{\mathrm{ER}}_p:=\|\mathbf{u}_{\mathrm{ER}}(p)\|_2^2/E_{\mathrm{C}}$,
respectively. Since
$\mathbf{u}_{\mathrm{ER}}(p)=m_p\mathbf{u}_{\mathrm{C}}(p)$, these shares
satisfy
\begin{equation}
\pi^{\mathrm{ER}}_p
=
m_p^2\pi^{\mathrm{C}}_p.
\label{eq:spatial_energy_share}
\end{equation}
Equation~\eqref{eq:spatial_energy_share} makes the redistribution explicit:
\method modifies the original spatial energy allocation only through the local
multipliers, while the total energy remains unchanged.

The local consistency scores and the shared scalar $\eta^\star$ play
complementary roles: the scores determine the relative spatial penalties,
while $\eta^\star$ enforces the common energy budget.

Importantly, the equal-energy constraint does not require every low-consistency
location to be attenuated or every high-consistency location to be amplified
relative to one. Several multipliers may lie on the same side of one, provided
that their energy-weighted adjustments collectively preserve
$E_{\mathrm{C}}$. Therefore, \method should be interpreted as a globally
coupled spatial reallocation rather than an independent pointwise suppression
rule.

The same optimization also yields an additional certificate showing that the
redistributed field places no more reliability-weighted energy on inconsistent
locations than the original Chord field. We provide the formal statement and
proof in Appendix~\ref{app:reliability_certificate}.

\begin{table*}[t]
\caption{\textbf{Editing results on PIE-Bench}~\citep{ju2024pnp}. MSE and LPIPS are scaled by $10^3$. Best, second-best, and third-best results are highlighted with \colorbox{bestcolor}{green}, \colorbox{secondcolor}{yellow}, and \colorbox{thirdcolor}{purple} backgrounds, respectively. T-/I-free: training-/inversion-free. Extended metrics appear in Appendix~\ref{app:experiments}.}
\label{tab:main}
\vspace{1.0ex}
\centering
\small
\resizebox{\linewidth}{!}{%
\setlength{\tabcolsep}{3.5pt}
\begin{tabular}{ll rrr rr cc rrrr}
\toprule
& & \multicolumn{3}{c}{Background Preservation} & \multicolumn{2}{c}{Target Alignment} & \multicolumn{2}{c}{Properties} & \multicolumn{4}{c}{Efficiency} \\
\cmidrule(lr){3-5}\cmidrule(lr){6-7}\cmidrule(lr){8-9}\cmidrule(lr){10-13}
Type & Method & PSNR$\uparrow$ & MSE$\downarrow$ & LPIPS$\downarrow$ & Whole$\uparrow$ & Edit$\uparrow$ & T-free & I-free & Steps$\downarrow$ & NFE$\downarrow$ & Time(s)$\downarrow$ & VRAM(M)$\downarrow$\\
\midrule
\multirow{5}{*}{Multi-step}
& DDIM + MasaCtrl (2023) & 21.25 & 8.58 & 106.59 & 24.13 & 21.13 & \cmark & \xmark & 50 & 100 & 55.20 & 12272\\
& Direct Inv. + MasaCtrl (2023) & 21.78 & 7.99 & 87.38 & 24.42 & 21.38 & \cmark & \xmark & 50 & 100 & 79.10 & 12272\\
& DDIM + PnP (2023) & 21.26 & 8.42 & 113.58 & 25.45 & 22.54 & \cmark & \xmark & 50 & 100 & 28.01 & \cthird{9262}\\
& Direct Inv. + PnP (2023) & 21.43 & 8.10 & 106.26 & 25.48 & 22.63 & \cmark & \xmark & 50 & 100 & 28.03 & \cthird{9262}\\
& FlowEdit (SD3) (2025) & 22.17 & 7.69 & 104.81 & \cbest{26.64} & \cbest{23.69} & \cmark & \cmark & \cthird{33} & 33 & 7.22 & 17140\\
\midrule
\multirow{3}{*}{Few-step}
& TurboEdit (SDXL-T) (2024) & 21.44 & 9.49 & 108.60 & 24.66 & 21.79 & \cmark & \cmark & \csecond{4} & \cthird{4} & 2.69 & 13826\\
& InfEdit (SD1.4) (2024) & \csecond{24.14} & 6.82 & \cbest{55.69} & 24.89 & 21.88 & \cmark & \cmark & \csecond{4} & \cthird{4} & 1.41 & \cbest{6502}\\
& InstantEdit (PeRFlow) (2025) & 23.80 & \cbest{4.21} & \csecond{60.92} & 24.97 & 21.82 & \cmark & \xmark & \csecond{4} & 8 & 1.30 & 16270\\
\midrule
\multirow{7}{*}{One-step}
& SwiftEdit (SB-v2) (2025) & 21.71 & 8.22 & 91.22 & 24.93 & 21.85 & \xmark & \xmark & \cbest{1} & \csecond{2} & 0.54 & 15060\\
& ChordEdit (SB-v2) (2026) & 22.04 & 7.13 & 111.22 & 25.12 & 22.58 & \cmark & \cmark & \cbest{1} & \csecond{2} & \cthird{0.38} & \csecond{6988}\\
& ChordEdit (w/o clean, SD-T) & \cthird{23.89} & \cthird{5.05} & 88.36 & 24.97 & 21.87 & \cmark & \cmark & \cbest{1} & \cbest{1} & \cbest{0.20} & \csecond{6988}\\
& ChordEdit (SD-Turbo) (2026) & 22.20 & 6.84 & 128.25 & \cthird{25.58} & \cthird{22.96} & \cmark & \cmark & \cbest{1} & \csecond{2} & \cthird{0.38} & \csecond{6988}\\
\cmidrule(lr){2-13}
& \textbf{\method (SB-v2)} & 22.48 & 6.62 & 105.10 & 25.26 & 22.74 & \cmark & \cmark & \cbest{1} & \csecond{2} & \csecond{0.37} & \csecond{6988}\\
& \textbf{\method (w/o clean, SD-T)} & \cbest{24.31} & \csecond{4.70} & \cthird{82.82} & 25.04 & 21.99 & \cmark & \cmark & \cbest{1} & \cbest{1} & \cbest{0.20} & \csecond{6988}\\
& \textbf{\method (SD-Turbo)} & 22.77 & 6.35 & 120.23 & \csecond{25.74} & \csecond{23.16} & \cmark & \cmark & \cbest{1} & \csecond{2} & \csecond{0.37} & \csecond{6988}\\
\bottomrule
\end{tabular}%
}
\vspace{0.5ex}
\end{table*}

\section{Experiments}

\subsection{Evaluation protocol}

\leadheading{Benchmarks.}
PIE-Bench~\citep{ju2024pnp} contains 700 masked editing examples in ten categories and provides the primary comparison with \base.
We additionally evaluate the original 100-example TEdBench~\citep{kawar2023imagic} as an external benchmark for text-guided real-image editing.
For ChordEdit and \method, we use the same target-only protocol with identical inputs, target captions, random seeds, and evaluation settings.

\leadheading{Metrics.}
On PIE-Bench, background preservation uses non-edited-region PSNR, MSE, SSIM, LPIPS~\citep{zhang2018unreasonable}, and structure distance; semantic alignment uses CLIP-Whole and CLIP-Edited~\citep{radford2021learning}.
On TEdBench, we report LPIPS for perceptual source preservation, CLIP-I for image-level source similarity, CLIP-T for target text--image similarity, and Dir-Sim for directional similarity.
Every ER run also records the dual interval, normalized dual residual, equal-energy residual, KKT residual, multiplier bounds, direction residual, and inequality~\eqref{eq:app_reliability_certificate}.

\leadheading{Implementation and statistics.}
We retain the ChordEdit configuration: one shared noise sample, $t=0.90$, $\delta=0.15$, $\lambda=1.0$, and cleanup time $t_c=0.30$.
For \method, we use $\mu=0.25$, $\kappa=1.0$, $\varepsilon_c=10^{-8}$, an energy threshold of $10^{-12}$, a dual tolerance of $10^{-6}$, and at most 32 bisection iterations.
These values are not tuned on the external benchmarks.
The comparisons between ChordEdit and \method are paired by source example.
Runtime and peak memory are measured with the same hardware, precision, batch size, and warm-up.

\subsection{Comparison on PIE-Bench}

\leadheading{Full PIE-Bench evaluation.}
Table~\ref{tab:main} reproduces the comparison from \base~Table~1 and adds \method under the same SD-Turbo configuration with $(\mu,\kappa)=(0.25,1.0)$, one shared noise sample, and all 700 PIE-Bench examples (seed 42).
The comparison covers all 700 PIE-Bench examples across the ten predefined editing categories.
As shown in Table~\ref{tab:main}, the primary benefit of \method lies in substantial gains in source content preservation.
Specifically, \method improves unedited-region PSNR from $22.20$ to $22.77$\,dB ($+0.57$\,dB), lowers MSE from $6.84$ to $6.35$, and reduces perceptual distortion LPIPS from $128.25$ to $120.23$ ($10^3\times$ scaling).
Detailed structural preservation metrics (e.g., SSIM and Structure Distance) exhibit concordant improvements in the appendix (Table~\ref{tab:pie-detail}).
Crucially, these preservation gains are accompanied by improved target semantic alignment: CLIP-Whole rises from $25.58$ to $25.74$, and CLIP-Edited from $22.96$ to $23.16$.
The two methods have the same total pre-cleanup control energy by construction. Thus, the gains cannot be attributed to uniformly reducing the control field; they are consistent with a more effective spatial allocation.

\begin{figure*}[t]
\centering
\includegraphics[width=\linewidth]{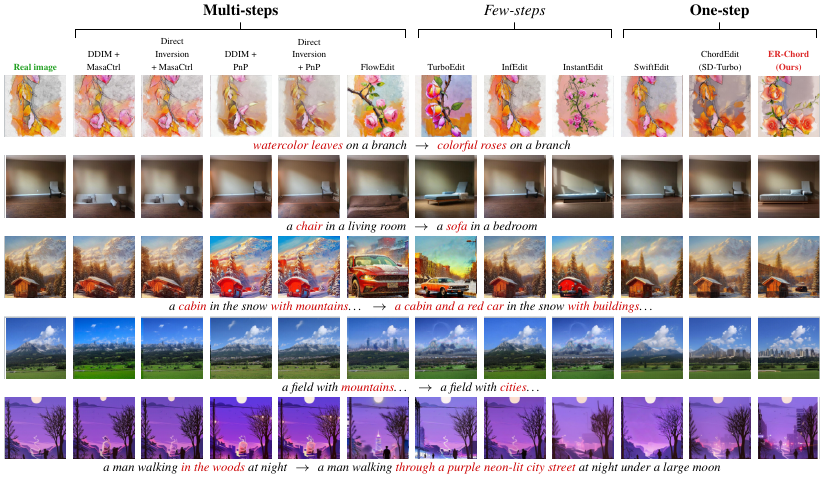}
\caption{\textbf{Qualitative comparison on PIE-Bench.} Real images are in the first column. Multi-step, few-step, and one-step methods are arranged horizontally. Prompts are noted under each row. ChordEdit and \method use the identical SD-Turbo backbone.}
\label{fig:pie-qualitative}
\end{figure*}

\leadheading{Qualitative comparison.}
In the examples shown in Figure~\ref{fig:pie-qualitative}, several multi-step and few-step editors alter the background or blur fine details. ChordEdit also exhibits visible background changes in some cases. With reliability-guided redistribution, \method produces cleaner target edits in these examples, including the sofa and mountain-to-building transformations, while retaining more of the source background.

\subsection{Ablations, Cross-Backbone Transfer, and Efficiency}

\leadheading{Transport vs.\ cleanup ablation.}
To understand how spatial redistribution interacts with the multi-stage pipeline, Table~\ref{tab:transport-cleanup} evaluates \method with and without the post-hoc cleanup stage.
In the no-cleanup setting ($\text{NFE}=1$), the evaluation directly isolates the spatial transport field.
Here, \method improves PSNR from $23.89$ to $24.31$\,dB and CLIP-Edited from $21.87$ to $21.99$ over the raw Chord field, confirming that reliability reallocation strengthens transport dynamics before cleanup.
With cleanup ($\text{NFE}=2$), \method maintains its advantage ($22.77$ vs.\ $22.20$\,dB PSNR; $23.16$ vs.\ $22.96$ CLIP-Edited), orthogonally complementing post-hoc temporal cleanup.

\leadheading{Core mechanism ablations.}
To isolate the necessity of equal-energy optimization, we test two baselines. First, a heuristic ``Reweight-and-Normalize (RN)'' directly scales the Chord field by $c_p$ and normalizes total energy, achieving only 22.40\,dB PSNR (vs.\ \method 22.77\,dB). Second, ``Spatial Permutation'' applies the exact optimal multipliers from \method but randomly permutes their spatial assignment. The PSNR drops to 21.05\,dB, supporting the importance of matching the multipliers to their original spatial locations.

\begin{table*}[t]
\centering
\begin{minipage}[t]{0.46\textwidth}
\caption{\textbf{Transport \& cleanup ablation.} PIE-Bench. NFE=1: transport; NFE=2: full. Bold: best.}
\label{tab:transport-cleanup}
\end{minipage}%
\hfill
\begin{minipage}[t]{0.51\textwidth}
\caption{\textbf{TEdBench comparison}~\citep{kawar2023imagic}. 100 real-image prompts. Bold: best.}
\label{tab:tedbench}
\end{minipage}

\begin{minipage}[t]{0.46\textwidth}
\centering
\scalebox{0.73}{%
\small
\renewcommand{\arraystretch}{0.93}
\setlength{\tabcolsep}{4pt}
\begin{tabular}{llrrr}
\toprule
\multicolumn{2}{c}{Configuration} & \multicolumn{2}{c}{PIE-Bench metrics} & \\
\cmidrule(lr){1-2}\cmidrule(lr){3-4}
Field & Cleanup & PSNR$\uparrow$ & CLIP-Edited$\uparrow$ & NFE\\
\midrule
Naive ($\delta=0$) & no & 21.89 & 20.83 & 1\\
Chord ($\delta=.15$) & no & 23.89 & 21.87 & 1\\
\rowcolor{bestcolor!40} \textbf{\method} & no & \textbf{24.31} & \textbf{21.99} & \textbf{1}\\
\midrule
Naive ($\delta=0$) & yes & 21.38 & 21.96 & 2\\
Chord ($\delta=.15$) & yes & 22.20 & 22.96 & 2\\
\rowcolor{bestcolor!40} \textbf{\method} & yes & \textbf{22.77} & \textbf{23.16} & 2\\
\bottomrule
\end{tabular}%
}
\end{minipage}%
\hfill
\begin{minipage}[t]{0.51\textwidth}
\centering
\scalebox{0.73}{%
\small
\renewcommand{\arraystretch}{0.93}
\setlength{\tabcolsep}{2.5pt}
\begin{tabular}{llcccc}
\toprule
Method & Paradigm & LPIPS$\downarrow$ & CLIP-I$\uparrow$ & CLIP-T$\uparrow$ & Dir-Sim$\uparrow$ \\
\midrule
SDEdit & Multi-step & 0.30 & 0.85 & 0.22 & 0.15 \\
Pix2Pix-Zero & Multi-step & 0.29 & 0.84 & 0.23 & 0.14 \\
Imagic & Optimization & 0.52 & 0.86 & 0.25 & 0.20 \\
LEDITS++ & Inversion & 0.23 & 0.87 & 0.26 & 0.21 \\
FlowEdit & Flow & 0.22 & 0.89 & 0.26 & 0.22 \\
\midrule
ChordEdit & One-step & 0.21 & 0.91 & 0.28 & \textbf{0.25} \\
\rowcolor{bestcolor!40} \textbf{\method (Ours)} & One-step & \textbf{0.20} & \textbf{0.92} & \textbf{0.29} & \textbf{0.25} \\
\bottomrule
\end{tabular}%
}
\end{minipage}
\end{table*}

\leadheading{Transfer to TEdBench.}
To evaluate out-of-domain generalization, we benchmark on 100 diverse real-image editing cases in TEdBench~\citep{kawar2023imagic} featuring complex attribute modifications and diverse natural scenes.
As shown in Table~\ref{tab:tedbench}, \method improves upon ChordEdit in both perceptual source preservation (LPIPS decreases from $0.21$ to $0.20$) and image fidelity (CLIP-I increases from $0.91$ to $0.92$).
On the reported TEdBench metrics, \method also improves over the listed multi-step and optimization-based baselines while using one transport update. These results suggest transfer to the tested scenes without benchmark-specific tuning.

\begin{wraptable}[8]{r}{0.48\textwidth}
\vspace{-2.0em}
\caption{\textbf{Cross-backbone generalization.} Results across backbones. Bold: best per backbone.}
\label{tab:backbones}
\vspace{0.1ex}
\centering
\resizebox{\linewidth}{!}{%
\renewcommand{\arraystretch}{0.94}%
\setlength{\tabcolsep}{4.5pt}%
\begin{tabular}{llcc}
\toprule
Backbone & Method & PSNR$\uparrow$ & CLIP-Edited$\uparrow$\\
\midrule
\multirow{3}{*}{InstaFlow}
  & Naive & 22.05 & 20.19 \\
  & ChordEdit & 23.05 & 21.39 \\
\rowcolor{bestcolor!40}
  & \textbf{\method (Ours)} & \textbf{23.40} & \textbf{21.54} \\
\midrule
\multirow{3}{*}{SwiftBrush-v2}
  & Naive & 20.52 & 21.06 \\
  & ChordEdit & 22.04 & 22.58 \\
\rowcolor{bestcolor!40}
  & \textbf{\method (Ours)} & \textbf{22.48} & \textbf{22.74} \\
\midrule
\multirow{3}{*}{SD-Turbo}
  & Naive & 21.38 & 21.96 \\
  & ChordEdit & 22.20 & 22.96 \\
\rowcolor{bestcolor!40}
  & \textbf{\method (Ours)} & \textbf{22.77} & \textbf{23.16} \\
\bottomrule
\end{tabular}%
}
\end{wraptable}
\leadheading{Cross-backbone generalization.}
To assess transfer across backbones, Table~\ref{tab:backbones} evaluates \method with InstaFlow, SwiftBrush-v2, and SD-Turbo.
With fixed hyperparameters ($(\mu,\kappa)=(0.25,1.0)$), \method improves PSNR over ChordEdit by $+0.35$\,dB on InstaFlow, $+0.44$\,dB on SwiftBrush-v2, and $+0.57$\,dB on SD-Turbo. CLIP-Edited also improves on all three tested backbones.
\vspace{0.8em}
\leadheading{Computational overhead and convergence.}
Because \method reuses the ChordEdit endpoints ($\mathbf{R}_t, \mathbf{R}_{t-\delta}$), it adds no network evaluations relative to ChordEdit. The full pipeline uses $\text{NFE}=2$ with cleanup and $\text{NFE}=1$ without it, and the measured peak memory is unchanged at 6988\,MiB.
The scalar bisection terminates within 32 iterations and takes $<0.01$\,s CPU time. The measured full-pipeline runtimes are similar: $0.37$\,s/img for \method and $0.38$\,s/img for ChordEdit.

\section{Conclusion}
This paper investigates spatial magnitude allocation in training-free, one-step image editing. Existing transport methods stabilize updates across noise levels using temporal weights shared across spatial locations, leaving local reliability unaccounted for in the allocation of editing strength. We propose \method, which uses endpoint agreement to redistribute control energy under an exact pre-cleanup energy constraint. The resulting field preserves the original local directions and has a unique global solution determined by a scalar root, with no additional denoiser evaluations relative to ChordEdit. Experiments on PIE-Bench, TEdBench, and three generative backbones (SD-Turbo, SwiftBrush-v2, and InstaFlow) show improved source preservation alongside target alignment in the tested settings. These results suggest that spatial allocation is a useful complement to temporal stabilization. The equal-energy constraint prevents a trivial improvement from globally weakening the edit: energy reduced in one location must be reallocated to another. The endpoint-consistency signal determines that reallocation, while a shared scalar root keeps the computation small. Because the method rescales the Chord field without changing its local directions, it cannot correct an erroneous direction supplied by the underlying field. Testing how this dependence changes across additional editing tasks and backbones is a useful direction for future work.

\subsubsection*{AI Use Statement}
In this work, we used generative AI tools to improve the clarity, grammar, and presentation of the manuscript. All AI-assisted content was carefully reviewed and verified by the authors. All scientific ideas, methodological designs, experimental results, and conclusions were developed and validated by the authors. We take full responsibility for the final content of this work.

\vspace{1.0ex}
{\small
\setlength{\bibsep}{3.0pt}

\bibliography{sources/reference}
\bibliographystyle{configuration/conference}

\newpage
\appendix
\onecolumn

\section{Mathematical Derivations and Proofs}
\label{app:math}
\label{sec:appendix_math_analysis}

This appendix provides the mathematical details supporting the construction in
Section~\ref{sec:method}. We first derive the temporally aggregated Chord field
used as the starting point of our method. We then establish the existence,
uniqueness, and global optimality of the equal-energy redistribution solution,
followed by the properties implied by its closed-form structure. These results
complete the technical arguments omitted from the main text while preserving
the method-level presentation in Section~\ref{sec:method}.

\subsection{Derivation of the Temporally Aggregated Chord Field}
\label{app:chord_derivation}

For completeness, we derive the temporally aggregated Chord field used in
Section~\ref{sec:method}. Following ChordEdit~\cite{lu2026chordedit}, let
$\widehat{\mathbf{u}}_{t-\delta}$ denote the endpoint control estimate entering
the temporal surrogate at time $t-\delta$. The Chord estimator considers the
quadratic objective
\begin{equation}
\Phi_t(\mathbf{u};\mathbf{x}_{\tau})
=
t\|\mathbf{u}-\widehat{\mathbf{u}}_{t-\delta}\|_F^2
+
\int_{t-\delta}^{t}
\|\mathbf{u}-\mathbf{R}_{\xi}\|_F^2
\,\mathrm{d}\xi,
\label{eq:app_chord_surrogate}
\end{equation}
where $\mathbf{R}_{\xi}:=\mathbf{R}(\mathbf{x}_{\tau},\xi)$ denotes the
source-to-target proxy field at noise time $\xi$.

The objective in \eqref{eq:app_chord_surrogate} is strictly convex with
respect to $\mathbf{u}$ because its Hessian is
$2(t+\delta)\mathbf{I}\succ 0$. Therefore, its unique minimizer is
characterized by the first-order condition
\begin{equation}
t\left(
\mathbf{u}
-
\widehat{\mathbf{u}}_{t-\delta}
\right)
+
\int_{t-\delta}^{t}
\left(
\mathbf{u}
-
\mathbf{R}_{\xi}
\right)
\,\mathrm{d}\xi
=
\mathbf{0}.
\label{eq:app_chord_stationarity}
\end{equation}
Since the integration interval has length $\delta$, rearranging
\eqref{eq:app_chord_stationarity} gives the unique minimizer
\begin{equation}
\mathbf{u}_{t}^{\star}
=
\frac{t}{t+\delta}
\widehat{\mathbf{u}}_{t-\delta}
+
\frac{1}{t+\delta}
\int_{t-\delta}^{t}
\mathbf{R}_{\xi}
\,\mathrm{d}\xi.
\label{eq:app_chord_exact_minimizer}
\end{equation}

ChordEdit then approximates the previous endpoint estimate by
$\widehat{\mathbf{u}}_{t-\delta}\approx\mathbf{R}_{t-\delta}$ and evaluates
the temporal integral using the endpoint approximation
$\int_{t-\delta}^{t}\mathbf{R}_{\xi}\,\mathrm{d}\xi
\approx\delta\mathbf{R}_{t}$. Substituting these approximations into
\eqref{eq:app_chord_exact_minimizer} yields
\begin{equation}
\mathbf{u}_{\mathrm{C}}
=
\frac{t}{t+\delta}\mathbf{R}_{t-\delta}
+
\frac{\delta}{t+\delta}\mathbf{R}_{t}.
\label{eq:app_chord_field}
\end{equation}
Restricting \eqref{eq:app_chord_field} to any spatial location
$p\in\Omega$ recovers the Chord field used in
\eqref{eq:chord_field}. In particular, the coefficients
$t/(t+\delta)$ and $\delta/(t+\delta)$ are determined solely by the temporal
construction and are independent of $p$.

\subsection{Proof of the Unique Global Solution}
\label{app:er_solution_proof}

We now prove Theorem~\ref{thm:er_solution}. Recall from
Section~\ref{sec:closed_form} that
$a_p:=1+\mu(1-c_p)^2$. Since $c_p\in[0,1]$ and $0\leq\mu<1$, we have
$1\leq a_p\leq 1+\mu$ for every $p\in\Omega$.

\begin{proof}[Proof of Theorem~\ref{thm:er_solution}]
Let $J(\mathbf{u})$ denote the objective in
\eqref{eq:er_optimization}. Introducing a scalar Lagrange multiplier $\eta$
for the equal-energy constraint, the Lagrangian is
\begin{equation}
\mathcal{L}(\mathbf{u},\eta)
=
J(\mathbf{u})
+
\frac{\eta}{2}
\left(
\|\mathbf{u}\|_F^2-E_{\mathrm{C}}
\right).
\label{eq:app_er_lagrangian}
\end{equation}
The stationarity condition with respect to the local vector
$\mathbf{u}(p)$ is
\begin{equation}
(a_p+\eta)\mathbf{u}(p)
=
\mathbf{u}_{\mathrm{C}}(p),
\qquad p\in\Omega.
\label{eq:app_er_stationarity}
\end{equation}
Hence, whenever $a_p+\eta>0$ for all $p\in\Omega$, any stationary point must
satisfy
\begin{equation}
\mathbf{u}(p;\eta)
=
\frac{\mathbf{u}_{\mathrm{C}}(p)}
{a_p+\eta}.
\label{eq:app_er_candidate}
\end{equation}

It remains to determine the scalar $\eta$ from the equal-energy constraint.
Substituting \eqref{eq:app_er_candidate} into
$\|\mathbf{u}\|_F^2=E_{\mathrm{C}}$ gives
\begin{equation}
\phi(\eta)
:=
\sum_{p\in\Omega}
\frac{\|\mathbf{u}_{\mathrm{C}}(p)\|_2^2}
{(a_p+\eta)^2}
-
E_{\mathrm{C}}
=
0.
\label{eq:app_er_root}
\end{equation}
For any $\eta\in[-\mu,0]$, the bound $a_p\geq 1$ implies
$a_p+\eta\geq 1-\mu>0$. Therefore, $\phi$ is well defined and continuously
differentiable on $[-\mu,0]$. Moreover,
\begin{equation}
\phi'(\eta)
=
-2
\sum_{p\in\Omega}
\frac{\|\mathbf{u}_{\mathrm{C}}(p)\|_2^2}
{(a_p+\eta)^3}
<
0,
\label{eq:app_er_root_derivative}
\end{equation}
where the strict inequality follows from $E_{\mathrm{C}}>0$. Hence,
$\phi$ is strictly decreasing on this interval.

At the right endpoint, $a_p\geq 1$ gives
\begin{equation}
\phi(0)
=
\sum_{p\in\Omega}
\frac{\|\mathbf{u}_{\mathrm{C}}(p)\|_2^2}
{a_p^2}
-
E_{\mathrm{C}}
\leq
0.
\label{eq:app_er_right_endpoint}
\end{equation}
At the left endpoint, $a_p\leq 1+\mu$ gives
$a_p-\mu\leq 1$, and therefore
\begin{equation}
\phi(-\mu)
=
\sum_{p\in\Omega}
\frac{\|\mathbf{u}_{\mathrm{C}}(p)\|_2^2}
{(a_p-\mu)^2}
-
E_{\mathrm{C}}
\geq
0.
\label{eq:app_er_left_endpoint}
\end{equation}
The intermediate value theorem and the strict monotonicity of $\phi$ therefore
imply the existence of a unique
$\eta^\star\in[-\mu,0]$ satisfying $\phi(\eta^\star)=0$.
Substituting $\eta^\star$ into \eqref{eq:app_er_candidate} yields
\eqref{eq:er_closed_form}, and \eqref{eq:app_er_root} reduces to
\eqref{eq:dual_root}.

We finally establish global optimality over the complete feasible set, which
is the energy sphere $\|\mathbf{u}\|_F^2=E_{\mathrm{C}}$. Let
$\mathbf{A}$ denote the block-diagonal matrix whose block associated with
location $p$ is $a_p\mathbf{I}$. Then
\eqref{eq:app_er_stationarity} at $\eta=\eta^\star$ can be written as
\begin{equation}
(\mathbf{A}+\eta^\star\mathbf{I})
\mathbf{u}_{\mathrm{ER}}
=
\mathbf{u}_{\mathrm{C}}.
\label{eq:app_er_matrix_stationarity}
\end{equation}
Consider any feasible field $\mathbf{z}$ and define
$\mathbf{d}:=\mathbf{z}-\mathbf{u}_{\mathrm{ER}}$.
Since both $\mathbf{z}$ and $\mathbf{u}_{\mathrm{ER}}$ have energy
$E_{\mathrm{C}}$, we have
$2\langle\mathbf{u}_{\mathrm{ER}},\mathbf{d}\rangle_F
+\|\mathbf{d}\|_F^2=0$.
Using this identity together with
\eqref{eq:app_er_matrix_stationarity}, the objective difference satisfies
\begin{equation}
J(\mathbf{z})-J(\mathbf{u}_{\mathrm{ER}})
=
\frac{1}{2}
\left\langle
\mathbf{d},
(\mathbf{A}+\eta^\star\mathbf{I})\mathbf{d}
\right\rangle_F.
\label{eq:app_er_objective_gap}
\end{equation}
Since $a_p\geq 1$ and $\eta^\star\geq-\mu$,
\begin{equation}
\mathbf{A}+\eta^\star\mathbf{I}
\succeq
(1-\mu)\mathbf{I}
\succ
\mathbf{0}.
\label{eq:app_er_positive_definite}
\end{equation}
Consequently,
\begin{equation}
J(\mathbf{z})-J(\mathbf{u}_{\mathrm{ER}})
\geq
\frac{1-\mu}{2}
\|\mathbf{d}\|_F^2.
\label{eq:app_er_global_gap}
\end{equation}
The right-hand side is strictly positive whenever
$\mathbf{z}\neq\mathbf{u}_{\mathrm{ER}}$. Hence,
$\mathbf{u}_{\mathrm{ER}}$ is the unique global minimizer of
\eqref{eq:er_optimization}.
\end{proof}

\subsection{Properties of the Closed-Form Redistribution}
\label{app:er_properties}

We next establish the properties of the closed-form solution stated in
Corollary~\ref{cor:er_properties}. Recall that
$m_p:=(a_p+\eta^\star)^{-1}$, where
$a_p=1+\mu(1-c_p)^2$ and $\eta^\star\in[-\mu,0]$.

\begin{proof}[Proof of Corollary~\ref{cor:er_properties}]
The exact energy preservation
\begin{equation}
\|\mathbf{u}_{\mathrm{ER}}\|_F^2
=
\|\mathbf{u}_{\mathrm{C}}\|_F^2
=
E_{\mathrm{C}}
\label{eq:app_energy_preservation}
\end{equation}
follows directly from the equality constraint in
\eqref{eq:er_optimization} and the feasibility of the solution constructed in
Theorem~\ref{thm:er_solution}.

Since $c_p\in[0,1]$, we have $1\leq a_p\leq 1+\mu$. Together with
$\eta^\star\in[-\mu,0]$, this gives
$a_p+\eta^\star\geq 1-\mu>0$. Hence $m_p>0$ for every $p\in\Omega$.
Therefore,
$\mathbf{u}_{\mathrm{ER}}(p)=m_p\mathbf{u}_{\mathrm{C}}(p)$ preserves the
direction of every nonzero local Chord vector.

The same bounds imply
$1-\mu\leq a_p+\eta^\star\leq 1+\mu$, and consequently
$1/(1+\mu)\leq m_p\leq 1/(1-\mu)$.

Finally, suppose that $\mu>0$ and $c_p>c_q$. Since
$(1-c_p)^2<(1-c_q)^2$, we have $a_p<a_q$. The scalar
$\eta^\star$ is shared across all spatial locations, so
$a_p+\eta^\star<a_q+\eta^\star$. Both denominators are positive, and therefore
$m_p>m_q$.
\end{proof}

The spatial energy relation used in Section~\ref{sec:closed_form} follows
immediately from the same multiplicative representation. Recall that
$\pi_p^{\mathrm{C}}
:=\|\mathbf{u}_{\mathrm{C}}(p)\|_2^2/E_{\mathrm{C}}$
and
$\pi_p^{\mathrm{ER}}
:=\|\mathbf{u}_{\mathrm{ER}}(p)\|_2^2/E_{\mathrm{C}}$.
Then
\begin{equation}
\pi_p^{\mathrm{ER}}
=
\frac{m_p^2\|\mathbf{u}_{\mathrm{C}}(p)\|_2^2}
{E_{\mathrm{C}}}
=
m_p^2\pi_p^{\mathrm{C}}.
\label{eq:app_spatial_energy_share}
\end{equation}
Hence the local multipliers act directly on the original spatial energy shares,
while \eqref{eq:app_energy_preservation} ensures that the shares remain
normalized over the entire field. In particular, the ordering
$c_p>c_q\Rightarrow m_p>m_q$ concerns the relative multiplicative adjustment
and does not imply
$\pi_p^{\mathrm{ER}}>\pi_q^{\mathrm{ER}}$, since the final shares also depend
on the original Chord allocation.

The closed form also recovers several limiting cases. If $\mu=0$, then
$a_p=1$ for every $p$, and the equal-energy condition yields
$\eta^\star=0$, so $\mathbf{u}_{\mathrm{ER}}=\mathbf{u}_{\mathrm{C}}$.
More generally, if $c_p$ is constant over the nonzero support of
$\mathbf{u}_{\mathrm{C}}$, all active locations share the same positive
multiplier. Exact energy preservation then forces this multiplier to equal
one. When $E_{\mathrm{C}}=0$, the constraint admits only the zero field, and
thus $\mathbf{u}_{\mathrm{ER}}=\mathbf{0}$.

\subsection{Reliability-Weighted Energy Certificate}
\label{app:reliability_certificate}

Beyond the pointwise properties of the closed-form solution, the optimization
objective also controls how much energy is assigned to locations with low
endpoint consistency. For any nonzero control field $\mathbf{u}$, define its
reliability-weighted energy fraction as
\begin{equation}
\Mismatch(\mathbf{u};\mathbf{c})
:=
\frac{
\sum_{p\in\Omega}
(1-c_p)^2
\|\mathbf{u}(p)\|_2^2
}{
\|\mathbf{u}\|_F^2
},
\label{eq:app_reliability_mismatch}
\end{equation}
where $\mathbf{c}$ denotes the collection of endpoint-consistency scores
$\{c_p\}_{p\in\Omega}$.

\begin{proposition}[Reliability-weighted energy certificate]
\label{prop:reliability_certificate}
Assume $\mu>0$ and $E_{\mathrm{C}}>0$. Then the solution of
\eqref{eq:er_optimization} satisfies
\begin{equation}
\Mismatch(\mathbf{u}_{\mathrm{ER}};\mathbf{c})
\leq
\Mismatch(\mathbf{u}_{\mathrm{C}};\mathbf{c})
-
\frac{
\|\mathbf{u}_{\mathrm{ER}}-\mathbf{u}_{\mathrm{C}}\|_F^2
}{
\mu E_{\mathrm{C}}
}
\leq
\Mismatch(\mathbf{u}_{\mathrm{C}};\mathbf{c}).
\label{eq:app_reliability_certificate}
\end{equation}
\end{proposition}

\begin{proof}
The original Chord field $\mathbf{u}_{\mathrm{C}}$ is feasible for
\eqref{eq:er_optimization} because
$\|\mathbf{u}_{\mathrm{C}}\|_F^2=E_{\mathrm{C}}$. By the global optimality of
$\mathbf{u}_{\mathrm{ER}}$ established in
Theorem~\ref{thm:er_solution},
\begin{equation}
J(\mathbf{u}_{\mathrm{ER}})
\leq
J(\mathbf{u}_{\mathrm{C}})
=
\frac{\mu}{2}
\sum_{p\in\Omega}
(1-c_p)^2
\|\mathbf{u}_{\mathrm{C}}(p)\|_2^2.
\label{eq:app_certificate_optimality}
\end{equation}
Expanding the left-hand side yields
\begin{equation}
\begin{aligned}
\frac{1}{2}
\|\mathbf{u}_{\mathrm{ER}}-\mathbf{u}_{\mathrm{C}}\|_F^2
+
\frac{\mu}{2}
\sum_{p\in\Omega}
(1-c_p)^2
\|\mathbf{u}_{\mathrm{ER}}(p)\|_2^2
\leq
\frac{\mu}{2}
\sum_{p\in\Omega}
(1-c_p)^2
\|\mathbf{u}_{\mathrm{C}}(p)\|_2^2.
\end{aligned}
\label{eq:app_certificate_expansion}
\end{equation}
Both $\mathbf{u}_{\mathrm{ER}}$ and $\mathbf{u}_{\mathrm{C}}$ have total
energy $E_{\mathrm{C}}$. Dividing
\eqref{eq:app_certificate_expansion} by $\mu E_{\mathrm{C}}/2$ and using
\eqref{eq:app_reliability_mismatch} gives
\begin{equation}
\Mismatch(\mathbf{u}_{\mathrm{ER}};\mathbf{c})
+
\frac{
\|\mathbf{u}_{\mathrm{ER}}-\mathbf{u}_{\mathrm{C}}\|_F^2
}{
\mu E_{\mathrm{C}}
}
\leq
\Mismatch(\mathbf{u}_{\mathrm{C}};\mathbf{c}),
\label{eq:app_certificate_intermediate}
\end{equation}
which establishes the first inequality in
\eqref{eq:app_reliability_certificate}. The second follows because the
additional term in \eqref{eq:app_certificate_intermediate} is nonnegative.
\end{proof}

\begin{proposition}[Bounds on the Dual Root]
\label{prop:dual_bounds}
Let $\eta^\star$ be the optimal dual root, and let
$\pi_p=\|\mathbf{u}_{\mathrm{C}}(p)\|_2^2/E_{\mathrm{C}}$.
The dual root satisfies
\begin{equation}
-\mu \Mismatch(\mathbf{u}_{\mathrm{C}};\mathbf{c}) \leq \eta^\star \leq 0.
\label{eq:dual_kkt_jensen}
\end{equation}
\end{proposition}

\begin{proof}
For the lower bound, apply Jensen's inequality to the strictly convex
function $f(x)=x^{-2}$ on $x>0$:
\begin{equation}
\begin{aligned}
1
&=\sum_p\pi_p\bigl(1+\mu(1-c_p)^2+\eta^\star\bigr)^{-2}\\
&\geq
\left[\sum_p\pi_p\bigl(1+\mu(1-c_p)^2+\eta^\star\bigr)\right]^{-2}\\
&=\bigl(1+\mu\Mismatch(\mathbf{u}_{\mathrm{C}};\mathbf{c})
+\eta^\star\bigr)^{-2}.
\end{aligned}
\end{equation}
The argument is positive, so the inequality implies
$\eta^\star\geq-\mu\Mismatch(\mathbf{u}_{\mathrm{C}};\mathbf{c})$.
The upper bound $\eta^\star\leq0$ follows from the endpoint argument in
the proof of Theorem~\ref{thm:er_solution}. These bounds can tighten the
bisection interval.
\end{proof}

Proposition~\ref{prop:reliability_certificate} concerns the
reliability-weighted control energy defined in
\eqref{eq:app_reliability_mismatch}. It does not by itself imply an
improvement in decoded image quality or semantic correctness.

\section{Implementation Details, Numerical Diagnostics, and Reproducibility}
\label{app:implementation}

\subsection{Algorithmic Flow and Finite-Precision Bisection}
For each input image, \method reuses the two Chord endpoint proxy fields
($\mathbf{R}_t,\mathbf{R}_{t-\delta}$) to construct the base field $\uc$.
It computes the channel-reduced signal and disagreement maps, forms the
consistency scores $c_p$ and weights $a_p$, solves for the scalar root
$\eta^\star$, and applies the spatial multipliers before the standard Chord
update and post-hoc cleanup. For active samples
($E_{\mathrm{C}}>\varepsilon_E$), the normalized dual and energy residuals are
$r_\phi=|\phi(\eta^\star)|/E_{\mathrm{C}}$ and
$r_E=|\|\uer\|_F^2-E_{\mathrm{C}}|/E_{\mathrm{C}}$, respectively.

Table~\ref{tab:certificates} reports numerical solver diagnostics across
all 700 PIE-Bench examples. The measured residuals meet the stated
tolerances, the multipliers remain within the theoretical bounds, and
no violations of the reliability-weighted energy inequality are observed.

\begin{table}[H]
\vspace{-1.2ex}
\caption{\textbf{Solver checks on PIE-Bench (700 examples, seed 42).} Residuals, multiplier bounds, direction preservation, and the mismatch inequality are computed on the realized \method field after dtype conversion.}
\label{tab:certificates}
\vspace{-0.5ex}
\centering
\small
\begin{tabular}{ll}
\toprule
Certificate & Measured (required)\\
\midrule
Active samples ($\Ec>\varepsilon_E$) & 700 of 700\\
Dual root $\eta^*$ (mean; min; max) & $-0.012$; $-0.155$; $-0.001$ (within $[-0.25,0]$)\\
Normalized dual residual $r_\phi$ (max) & $4.1\times10^{-7}$ ($\leq10^{-6}$)\\
Relative energy residual $r_E$ (max) & $5.3\times10^{-7}$\\
Normalized KKT residual (max) & $2.8\times10^{-8}$ ($\leq10^{-6}$)\\
Multiplier range (min; max) & $0.801$; $1.183$ (within $[0.8,4/3]$)\\
Minimum direction cosine & $0.99998$\\
Proxy mismatch before $\Mismatch(\uc;c)$ & 0.0520\\
Proxy mismatch after $\Mismatch(\uer;c)$ & 0.0473\\
Energy--reliability mismatch bound violations & 0\\
\bottomrule
\end{tabular}
\vspace{-2.0ex}
\end{table}

\vspace{-0.8ex}
\subsection{Experimental Setup and Reproducibility Protocol}
\label{app:reproducibility}
\vspace{-0.4ex}
All \method experiments strictly follow the ChordEdit~\citep{lu2026chordedit} protocol: one shared noise sample, $t=0.90$, $\delta=0.15$, update scale $\lambda=1.0$, and cleanup time $t_c=0.30$. For \method, we set $\mu=0.25$, $\kappa=1.0$, $\varepsilon_c=10^{-8}$, energy threshold $\varepsilon_E=10^{-12}$, and dual tolerance $10^{-6}$. These hyperparameters remain completely fixed across all PIE-Bench categories, TEdBench, and all three generative backbones (SD-Turbo, SwiftBrush-v2, InstaFlow) without any dataset-specific or backbone-specific tuning. Timing and memory benchmarks are measured on the same GPU and software environment with identical precision, batch size, warm-up, and synchronization; one-time model loading is excluded. Under this protocol, \method averages 0.37\,s/img (identical peak VRAM of 6988\,MiB to ChordEdit).

\vspace{-0.5ex}
\section{Additional Qualitative Results}
\label{sec:appendix_qualitative}

Figures~\ref{fig:appendix-additional-qualitative-1} and~\ref{fig:appendix-additional-qualitative-2} provide additional visual comparisons on PIE-Bench under matched random seeds across representative categories, supplementing Figure~\ref{fig:pie-qualitative}.
In these representative cases, some multi-step inversion methods (DDIM and
Direct Inversion with MasaCtrl or PnP) smooth non-target regions or alter
fine textures such as smoke plumes and mountain ridges. Some few-step
methods (TurboEdit, InfEdit, and InstantEdit) show under-editing or color
artifacts. The \method examples illustrate cleaner target changes with
better background preservation. The full pipeline uses one transport update
followed by cleanup and takes $0.37$\,s/img in the reported setup.

\clearpage
\begin{figure}[t]
\centering
\includegraphics[width=0.94\linewidth]{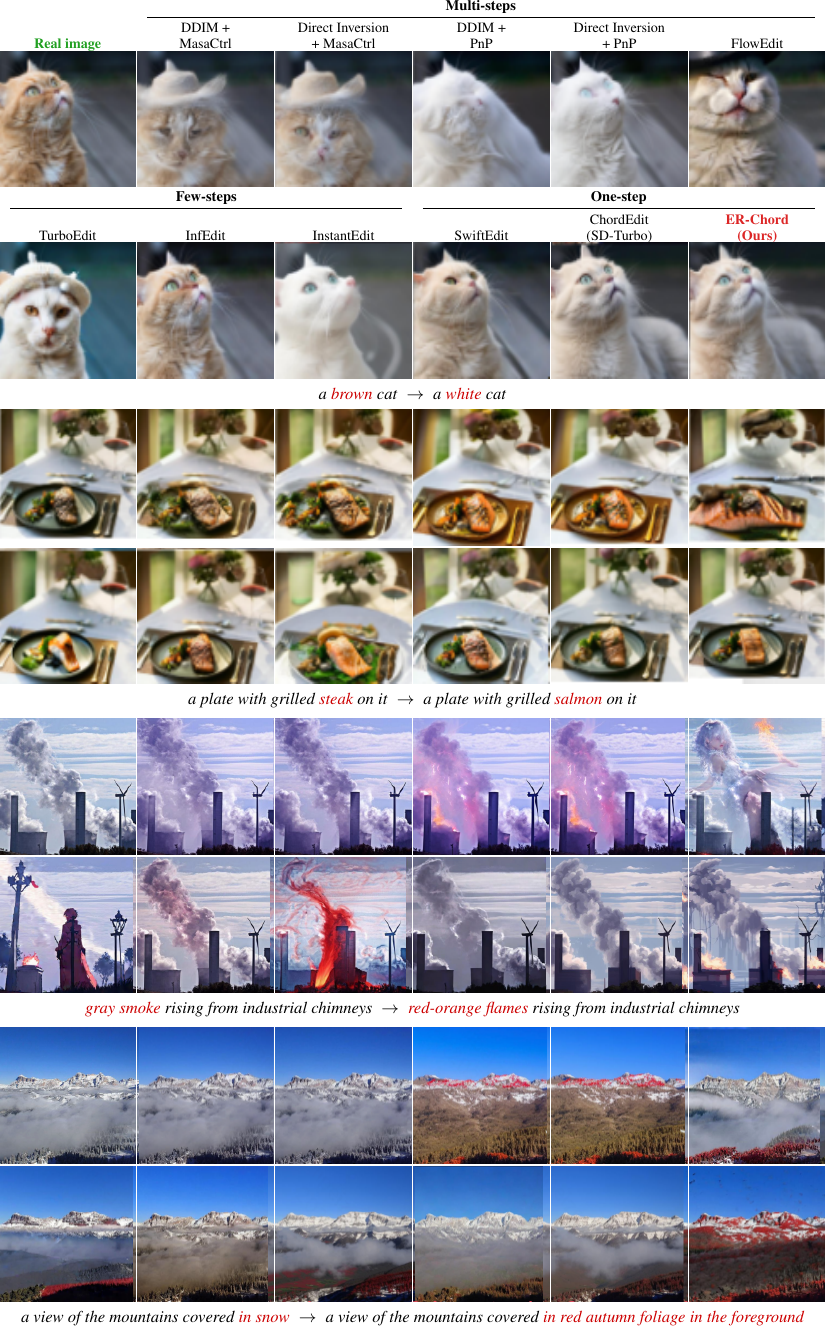}
\vspace{-1.0ex}
\caption{\textbf{Additional qualitative comparison on PIE-Bench (Part 1).} Extended visual comparisons across representative categories under matched random seeds, supplementing Figure~\ref{fig:pie-qualitative}.}
\label{fig:appendix-additional-qualitative-1}
\end{figure}

\clearpage
\begin{figure}[t]
\centering
\includegraphics[width=0.94\linewidth]{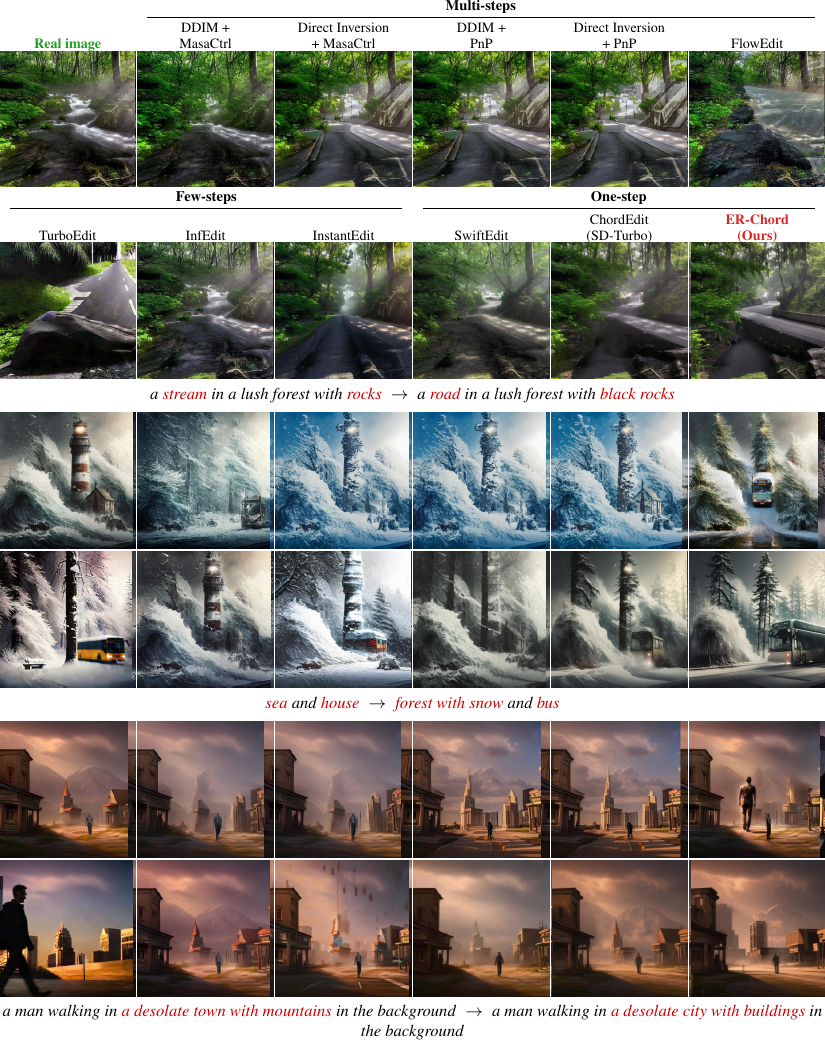}
\vspace{-1.0ex}
\caption{\textbf{Additional qualitative comparison on PIE-Bench (Part 2).} Extended visual comparisons across representative categories under matched random seeds, supplementing Figure~\ref{fig:pie-qualitative}.}
\label{fig:appendix-additional-qualitative-2}
\end{figure}

\clearpage

\section{Extended Quantitative Results and Notation}
\label{app:experiments}
\label{app:notation}
\vspace{-0.5ex}

Table~\ref{tab:pie-detail} reports extended quantitative results on all
700 PIE-Bench examples across the ten editing categories, supplementing
Table~\ref{tab:main}. Relative to ChordEdit, \method improves SSIM by
$0.96$ percentage points and reduces structure distance from
$16.58\times10^{-3}$ to $15.44\times10^{-3}$; CLIP-Whole and
CLIP-Edited also increase. Table~\ref{tab:symbols} summarizes the main
notation.

\begin{table}[!htbp]
\centering
\caption{\textbf{Additional PIE-Bench metrics (700 examples, seed 42).} ChordEdit values reproduce the extended metrics in \citet{lu2026chordedit}. MSE, LPIPS, and structure distance are scaled by $10^3$, and SSIM is reported as a percentage, following the source table. Bold denotes the better result.}
\label{tab:pie-detail}
\vspace{0.8ex}
\small
\begin{tabular}{lcc}
\toprule
Metric & \base & \method\\
\midrule
PSNR$\uparrow$ (dB) & 22.20 & \textbf{22.77}\\
MSE$\downarrow$ ($\times10^{-3}$) & 6.84 & \textbf{6.35}\\
SSIM$\uparrow$ ($\times10^{-2}$) & 75.91 & \textbf{76.87}\\
LPIPS$\downarrow$ ($\times10^{-3}$) & 128.25 & \textbf{120.23}\\
Structure dist.$\downarrow$ ($\times10^{-3}$) & 16.58 & \textbf{15.44}\\
CLIP-Whole$\uparrow$ & 25.58 & \textbf{25.74}\\
CLIP-Edited$\uparrow$ & 22.96 & \textbf{23.16}\\
\bottomrule
\end{tabular}

\vspace{4.0ex}

\caption{\textbf{Main notation used by \method.}}
\label{tab:symbols}
\vspace{0.8ex}
\small
\begin{tabularx}{\linewidth}{p{0.22\linewidth}X}
\toprule
Symbol & Meaning\\
\midrule
$\mathbf R_{t-\delta},\mathbf R_t$ & ChordEdit endpoint proxy fields at two nearby noise times.\\
$\uc$ & Original temporally smoothed Chord control field.\\
$p\in\Omega$ & Latent spatial position.\\
$s_p$ & Chord signal magnitude $\|\uc(p)\|_2$.\\
$d_p$ & Endpoint disagreement $\|\mathbf R_t(p)-\mathbf R_{t-\delta}(p)\|_2$.\\
$c_p$ & Endpoint-consistency score.\\
$\mu,\kappa$ & Reliability regularization strength and disagreement sensitivity.\\
$\Ec$ & Per-image Chord energy $\|\uc\|_F^2$.\\
$\eta^*$ & Dual variable enforcing equal energy.\\
$m_p$ & Local amplitude multiplier $(1+\mu(1-c_p)^2+\eta^*)^{-1}$.\\
$\uer$ & Equal-energy redistributed control field.\\
$\Mismatch(u;c)$ & Inconsistency-weighted fraction of control energy.\\
\bottomrule
\end{tabularx}
\end{table}

\end{document}